\documentclass[lettersize,journal]{IEEEtran}
\usepackage{amsmath,amsfonts,bm,amsthm}
\usepackage{algorithm}
\usepackage{array}
\usepackage{subcaption}
\usepackage{textcomp}
\usepackage{stfloats}
\usepackage{url}
\usepackage{verbatim}
\usepackage{graphicx}
\usepackage{cite}
\usepackage[capitalise]{cleveref}
\crefname{assumption}{Assumption}{Assumptions}
\usepackage{selinput}\SelectInputMappings{adieresis={ä},germandbls={ß}}
\usepackage{epsfig}
\usepackage{times}
\usepackage{float}
\usepackage{tikz,pgfplots}
\usepackage[english]{babel}
\usepackage{algpseudocode}
\usepackage{verbatim}
\usepackage{makecell}
\usepackage{multirow}
\usepackage{xcolor}
\usetikzlibrary{arrows,shapes,backgrounds,patterns,fadings,matrix,arrows,calc,
	intersections,decorations.markings,
	positioning,arrows.meta}
\usepgfplotslibrary{fillbetween}
\usepgfplotslibrary{statistics}
\pgfplotsset{width=5\columnwidth /5, compat = 1.13,
	height = 60\columnwidth /100, grid= major,
	legend cell align = left, ticklabel style = {font=\scriptsize},
	every axis label/.append style={font=\small},
	legend style = {font={\scriptsize}},title style={yshift=-7pt, font = \small} }

\newtheorem{corollary}{\bf{Corollary}}

\begin{document}
	
	\title{Asynchronous Distributed Gaussian Process Regression for Online Learning and Dynamical Systems: Complementary Document}
	
	\author{Zewen Yang, Xiaobing Dai, Sandra Hirche
		\thanks{Zewen Yang, Xiaobing Dai and Sandra Hirche are with the Chair of Information-oriented Control (ITR), School of Computation, Information and Technology (CIT), Technical University of Munich (TUM), 80333 Munich, Germany (email: zewen.yang@tum.de, xiaobing.dai@tum.de, hirche@tum.de).}
		\thanks{This is a complementary document for the paper titled with ``Asynchronous Distributed Gaussian Process Regression for Online Learning and Dynamical Systems'' \cite{anonymous2024asynchronous}.}% <-this % stops a space
	}
	
	\maketitle
	
	\section{Extended Introduction and Related Work}\label{sec_appendix_relatedWorks}
	
	In the realm of real-time online Gaussian Process (GP) regression, continuously collecting the training data becomes impractical for dynamic systems due to the constraints in physical storage space and the escalating computational burden \cite{dai2023can}. 
	Instead of employing the entire dataset for prediction, several approximation techniques prove instrumental. 
	These include sparse approximation methods such as sparse subset-of-data~\cite{snelson2005sparse,titsias2009variational} and sparse kernel methods~\cite{melkumyan2009sparse}, as well as projections/embedding approaches~\cite{nayebi2019framework}. 
	Moreover, local approximation methods, such as the naive local experts, the mixture of experts, and the product of experts, present viable alternatives.
	For an in-depth exploration of these approximation techniques, comprehensive discussions can be found in review papers~\cite{liuWhenGaussianProcess2020,jiang2022gaussian}. 
	
	\subsection{Local Approximation Gaussian Process}
	
	Prominently, cooperative learning within distributed systems falls under the purview of local approximation, where each subsystem, functioning as a computation node, assumes the role of an expert.
	A common instance in this context is the mixture of experts (MOE) method~\cite{trespMixturesGaussianProcesses2000,yuan2008variational,masoudniaMixtureExpertsLiterature2014}, where each model's prediction is weighted by the predefined factor, even with the selection of experts. 
	To alleviate this limitation, various techniques have been explored to enhance aggregate weights. The product of experts (POE), leveraging the posterior mean of GP models~\cite{hintonTrainingProductsExperts2002,ng2014hierarchical}, and the generalized product of experts (GPOE), which incorporates the difference in entropy between the prior and posterior~\cite{caoGeneralizedProductExperts2015}.
	Moreover, combining spare GP and barycenters, correlated POE is introduced in \cite{cohen2020healing,schurch2023correlated}.
	In contrast to the POE framework, both the Bayesian committee machine (BCM)~\cite{trespBayesianCommitteeMachine2000} and its robust variant, robust BCM (RBCM)~\cite{deisenrothDistributedGaussianProcesses2015}, integrate the GP prior. 
	However, these models exhibit limitations in sustaining consistent predictions as the training dataset undergoes continuous expansion. 
	Notably, the generalized RBCM, reliant on dataset sharing among models, introduces challenges in the context of distributed online learning~\cite{liu2018generalized} under limited communication or with constrained bandwidth.
	Additionally, the investigation into the nested pointwise aggregation of experts has been undertaken~\cite{rulliereNestedKrigingPredictions2018,9561566}. 
	Nevertheless, the application of pointwise aggregation across the entirety of the training dataset proves unattainable within distributed systems. 
	Therefore, the discernible increase in prediction time poses substantial practical challenges, particularly in real-time dynamical systems.

	\subsection{Agent-based Gaussian Process}
	
	Distributed learning finds prominent application in multi-agent systems (MASs), where each individual agent is treated as a central node, and their neighboring agents are considered computational nodes. 
	Consequently, joint predictions are aggregated within a cooperative learning framework to enhance individual predictions~\cite{yang2024cooperative,yang2024whom}.
	
	Several efforts have been dedicated to implementing distributed Gaussian Process (DGP) methodologies within MASs. 
	Effective strategies include employing the POE to refine aggregation weights~\cite{yangDistributedLearningConsensus2021} and applying the dynamical average consensus algorithm to accelerate joint prediction convergence~\cite{ledererCooperativeControlUncertain2023}. 
	Nevertheless, these endeavors demonstrate limitations in addressing concerns specific to online learning challenges. 
	In situations where GP models lack the capability to dynamically update with new observations or measurements, a potential constraint arises, as the data-driven model may encounter difficulties in adapting to a dynamic and evolving environment. 
	The works of \cite{hoang2019collective,yin2023learning,heDistributedOnlineSparse2023,dai2024cooperative,dai2024decentralized} have strategically addressed this crucial gap, emphasizing the necessity of integrating online data to ensure the adaptability of DGP in dynamic environments.
	
	While the data fusion method is characterized as decentralized in \cite{chen2015gaussian,ouyang2020gaussian}, it is crucial to emphasize that this approach entails the need for the central server to access a subset of data, implying the amalgamation of datasets. 
	However, this requirement poses challenges for its direct applicability in distributed systems. 
	Similarly, the nested pointwise aggregation approach, as discussed in \cite{kontoudis2021decentralized}, relies on data sharing, a condition that may not be conducive to certain distributed scenarios. 
	A strategy proposed by \cite{pillonetto2018distributed} introduces a finite-dimensional approximation of GP for MAS observation, aiming to reduce computational and communication demands. 
	However, this method mandates the collection of all data, despite a reduction in data dimensionality.
	The work in \cite{tavassolipour2019learning} introduced an algorithm designed to minimize communication costs while simultaneously achieving low distortion in the reconstruction of the Gram matrix within distributed systems. 
	However, these approaches still necessitate data sharing. 
	
	In the realm of Gaussian graphical models, \cite{jalali2021aggregating} introduces an innovative approach for aggregating Gaussian experts by identifying significant violations of conditional independence. 
	The dependency among experts is ascertained through a Gaussian graphical model, which provides the precision matrix leading to an improved aggregation strategy. 
	Meanwhile, \cite{10093993} presents a graph spectrum-based Gaussian process for predicting signals defined on graph nodes. 
	This approach incorporates a highly adaptive kernel that embraces various graph signal structures through a flexible polynomial function in the graph spectral domain. 
	
	In the scenario where each agent is equipped with functional sensors, \cite{xu2011mobile} devises a distributed navigation strategy, considering the limited communication range among agents. 
	Furthermore, \cite{jang2020multi} introduces an approach for active sensing model learning tailored for multi-agent coordination. 
	This methodology aims to discern the highest peak within an unknown environmental field while incorporating strategies for collision avoidance. 
	Similarly, in \cite{luo2018adaptive,jang2020multi,ding2024resource}, the online learning approaches for field mapping and environmental modeling are proposed through the application of distributed Gaussian processes.
	
	%%%%%%%%%%%%%%%%%%%%%%%%%%%%%%%%%%%%%%%%%%%%%%%%%%%%%%%%%%%%%%%%%%%%%%%%%%%%%%%%
	\section{Proofs for Prediction Performance}\label{sec_proofs}
	
	\subsection{Proof of Lemma 3}
	
	Using the triangular inequality, the delayed prediction error is bounded as
	\begin{align} \label{eqn_delayed_error_triangular_inequality}
		&| f(\bm{x}(t)) - \mu_i(\bm{x}(t_i^{k})) | \nonumber \\
		\le& | f(\bm{x}(t)) - f(\bm{x}(t_i^{k})) | + | f(\bm{x}(t_i^{k})) - \mu_i(\bm{x}(t_i^{k})) | \\
		\le& | f(\bm{x}(t)) - f(\bm{x}(t_i^{k})) | + \beta \sigma_i(\bm{x}(t_i^{k})), \nonumber 
	\end{align}
	where the second inequality is obtained by using the result in Lemma 2.
	Moreover, using the reproducing property of $\kappa$ \cite{hashimoto2022learning}, the first term in \eqref{eqn_delayed_error_triangular_inequality} becomes
	\begin{align} \label{eqn_delayed_error_reproducing_property}
		&| f(\bm{x}(t)) - f(\bm{x}(t_i^{k})) |^2 \\
		\le& \| f \|_{\kappa}^2  ( \kappa(\bm{x}(t), \bm{x}(t)) + \kappa(\bm{x}(t_i^{k}), \bm{x}(t_i^{k})) - 2 \kappa(\bm{x}(t), \bm{x}(t_i^{k})) ). \nonumber
	\end{align}
	Furthermore, considering $\| f \|_{\kappa} \le \Gamma$ and
	\begin{align}
		| \kappa(\cdot, \bm{x}(t)) - \kappa(\cdot, \bm{x}(t_i^{k})) | \le L_{\kappa} d(\bm{x}(t) - \bm{x}(t_i^{k})),
	\end{align}
	inequality \eqref{eqn_delayed_error_reproducing_property} is further bounded by
	\begin{align} \label{eqn_delayed_error_f}
		| f(\bm{x}(t)) - f(\bm{x}(t_i^{k})) |^2 \le& 2 \| f \|_{\kappa}^2  L_{\kappa} d(\bm{x}(t) - \bm{x}(t_i^{k})) \\
		\le& L_f^2 d(\bm{x}(t) - \bm{x}(t_i^{k})). \nonumber
	\end{align}
	Apply \eqref{eqn_delayed_error_f} into \eqref{eqn_delayed_error_triangular_inequality}, then the boundness of the delayed prediction error as $\eta_i^{k}(t)$ in the lemma is derived.
	
	\subsection{Proof of Theorem 7}
	
	With the design of aggregation weights $\omega_k^i(t)$ and $\omega_m(t)$, it is easy to see
	\begin{align}
		\sum_{i = 1}^M \sum_{k = 0}^{\bar{k}_i(t)} \omega_k^i(t) + \omega_m(t) = 1.
	\end{align}
	Combining with the triangular inequality, the prediction error is bounded by
	\begin{align}
		| f(\bm{x}(t)) - \hat{f}(\bm{x}(t)) | \le& \sum_{i = 1}^M \sum_{k = 0}^{\bar{k}_i(t)} \omega_k^i(t) | f(\bm{x}(t)) - \mu_i(\bm{x}(t_i^{k})) | \nonumber \\
		&+ \omega_m(t) | f(\bm{x}(t)) - m(\bm{x}(t)) | \\
		\le& \sum_{i = 1}^M \sum_{k = 0}^{\bar{k}_i(t)} \omega_k^i(t) \eta_i^{k}(t) + \omega_m(t) \beta \sigma_f, \nonumber
	\end{align}
	where the second inequality is derived using the result in Lemma 3 and $| f(\bm{x}(t)) - m(\bm{x}(t)) | \le \beta \sigma_f$.
	Note that maximal $\bar{\mathfrak{I}}$ previous predictions will be used, combining with the deterministic error bound for the prior mean, the second inequality holds using union bound.
	Additionally, applying the expression of aggregation weights, the prediction error is further bounded by
	\begin{align} \label{eqn_aggregated_error_bound_end}
		&| f(\bm{x}(t)) - \hat{f}(\bm{x}(t)) | \\
		\le& \omega^2(t) \Big( \sum_{i = 1}^M \sum_{k = 0}^{\bar{k}_i(t)} \rho_k^i(t) (\eta_i^{k}(t))^{-1} + (1 - \rho(t)) (\beta \sigma_f)^{-1} \Big) \nonumber
	\end{align}
	with
	\begin{equation}
		\rho(t) = \sum_{i = 1}^M \sum_{k = 0}^{\bar{k}_i(t)} \rho_k^i(t).
	\end{equation}
	Then, \eqref{eqn_aggregated_error_bound_end} is further bounded by
	\begin{align}
		&| f(\bm{x}(t)) - \hat{f}(\bm{x}(t)) | \nonumber \\
		\le& \frac{\sum_{i = 1}^M \sum_{k = 0}^{\bar{k}_i(t)} \rho_k^i(t) (\eta_i^{k}(t))^{-1} + (1 - \rho(t)) (\beta \sigma_f)^{-1}}{ \sum_{i = 1}^M \sum_{k = 0}^{\bar{k}_i(t)} \rho_k^i(t) (\eta_i^{k}(t))^{-2} + (1 - \rho(t)) (\beta \sigma_f)^{-2} } \\
		=& \frac{\sum_{i = 1}^M \sum_{k = 0}^{\bar{k}_i(t)} \rho_k^i(t) ((\eta_i^{k}(t))^{-1} - (\beta \sigma_f)^{-1}) + (\beta \sigma_f)^{-1} }{ \sum_{i = 1}^M \sum_{k = 0}^{\bar{k}_i(t)} \rho_k^i(t) ((\eta_i^{k}(t))^{-2} - (\beta \sigma_f)^{-2}) + (\beta \sigma_f)^{-2} }. \nonumber
	\end{align}
	Considering the definition of $\rho_k^i(t)$, it has $\rho_k^i(t) = 0$ if $\eta_i^{k}(t) \ge \beta \sigma_f$, such that $(\eta_i^{k}(t))^{-1} - (\beta \sigma_f)^{-1} > 0$.
	Apply Cauchy-Schwarz inequality, then
	\begin{align}
		&\left( \sum_{i = 1}^M \sum_{k=0}^{\bar{k}_i(t)} \rho_k^i(t) ((\eta_i^{k}(t))^{-1} - (\beta \sigma_f)^{-1}) + (\beta \sigma_f)^{-1} \right)^2 \nonumber \\
		\le& \left(\sum_{i = 1}^M \sum_{k = 0}^{\bar{k}_i(t)} \rho_k^i(t) + 1 - \rho(t) \right) \\
		&\times \left( \sum_{i = 1}^M \sum_{k = 0}^{\bar{k}_i(t)} \rho_k^i(t) ((\eta_i^{k}(t))^{-1} - (\beta \sigma_f)^{-1})^2 + (\beta \sigma_f)^{-2} \right) \nonumber \\
		=& \sum_{i = 1}^M \sum_{k = 0}^{\bar{k}_i(t)} \rho_k^i(t) ((\eta_i^{k}(t))^{-1} - (\beta \sigma_f)^{-1})^2 + (\beta \sigma_f)^{-2}. \nonumber
	\end{align}
	Additionally, considering
	\begin{align}
		&\big((\eta_i^k(t))^{-1} - (\beta \sigma_f)^{-1}\big)^2 \\
		\le& \big((\eta_i^{k}(t))^{-1} - (\beta \sigma_f)^{-1}\big) \big((\eta_i^{k}(t))^{-1} + (\beta \sigma_f)^{-1}\big) \\
		=& (\eta_i^{k}(t))^{-2} - (\beta \sigma_f)^{-2}, \nonumber
	\end{align}
	one has
	\begin{align}
		&\bigg(\sum_{i = 1}^M \sum_{k = 0}^{\bar{k}_i(t)} \rho_k^i(t) \big((\eta_i^{k}(t))^{-1} - (\beta \sigma_f)^{-1}\big) + (\beta \sigma_f)^{-1}\bigg)^2 \nonumber \\
		\le& \sum_{i = 1}^M \sum_{k = 0}^{\bar{k}_i(t)} \rho_k^i(t) \big((\eta_i^{k}(t))^{-2} - (\beta \sigma_f)^{-2}\big) + (\beta \sigma_f)^{-2}. 
	\end{align}
	Therefore, the prediction error is bounded by 
	\begin{align}
		&| f(\bm{x}(t)) - \hat{f}(\bm{x}(t)) | \\
		\le& \frac{ \sqrt{\sum_{i = 1}^M \sum_{k = 0}^{\bar{k}_i(t)} \rho_k^i(t) \big((\eta_i^{k}(t))^{-2} - (\beta \sigma_f)^{-2}\big) + (\beta \sigma_f)^{-2}} }{ \sum_{i = 1}^M \sum_{k = 0}^{\bar{k}_i(t)} \rho_k^i(t) ((\eta_i^{k}(t))^{-2} - (\beta \sigma_f)^{-2}) + (\beta \sigma_f)^{-2} } \nonumber\\
		=& \left( \sum_{i = 1}^M \sum_{k = 0}^{\bar{k}_i(t)} \rho_k^i(t) \big((\eta_i^{k}(t))^{-2} - (\beta \sigma_f)^{-2}\big) + (\beta \sigma_f)^{-2} \right)^{-1/2} \nonumber \\
		=& \omega(t), 
	\end{align}
	considering the definition of $\omega(t)$ in (12) in the main context.
	
	\subsection{Proof of Corollary 8}
	
	Considering the monotonic increasing of $\omega(t)$ w.r.t $\eta_i^{k}(t)$ for $\forall i = 1,\cdots,M$ from (12) in the main context and recalling $\eta_i^{k}(t) < \beta \sigma_f$ from the management of information set in Algorithm 1, the time-varying prediction bound $\omega(t)$ is further bounded by
	\begin{align}
		\omega(t) \le& \left( \sum_{i = 1}^M \sum_{k = 0}^{\bar{k}_i(t)} \rho_k^i(t) \big((\beta \sigma_f)^{-2} - (\beta \sigma_f)^{-2}\big) + (\beta \sigma_f)^{-2} \right)^{-\frac{1}{2}} \nonumber \\
		=& \beta \sigma_f,
	\end{align}
	which concludes the proof.
	
	\section{Proof for Control Performance in Theorem 9}
	
	Recall the error dynamics in (23) in the main context as
	\begin{align}
		\dot{\bm{e}}(t) = \boldsymbol{A} \bm{e}(t) + \bm{b} (f(\bm{x}(t)) - \hat{f}(\bm{x}(t))).
	\end{align}
	We have the solution of $\boldsymbol{e}$ written as
	\begin{align}\label{eq_eSolution}
		\boldsymbol{e}\left(t\right) =&\exp\left({\boldsymbol{A}\left(t-t_0\right)}\right)\boldsymbol{e}\left(t_0\right) \nonumber \\
		&+ \int_{t_0}^{t}{\exp\left({\boldsymbol{A}\left(t-\tau\right)}\right)\boldsymbol{b}\left(f\left(\boldsymbol{x}\left(\tau\right)\right)-\hat{f}\left(\boldsymbol{x}\left(\tau\right)\right)\right)d\tau} \nonumber \\
		\leq& \left \|\exp\left({\boldsymbol{A}\left(t-t_0\right)}\right)\right \|\|\boldsymbol{e}\left(t_0\right)\| \\
		&+ \int_{t_0}^{t}{\left\|\exp\left({\boldsymbol{A}\left(t-\tau\right)}\right)\right \|\left\|f\left(\boldsymbol{x}\left(\tau\right)\right)-\hat{f}\left(\boldsymbol{x}\left(\tau\right)\right)\right\|d\tau}. \nonumber
	\end{align}
	Using the eigenvalue decomposition, the matrix $\boldsymbol{A}$ defined in (22) in the main context is formulated as
	\begin{align}
		\boldsymbol{A} = \bm{Q} \bm{\Lambda} \bm{Q}^{-1},
	\end{align}
	where $\boldsymbol{\Lambda} = \text{diag}( \Lambda _1,\cdots,\Lambda_n)$ is a diagonal matrix composed of all eigenvalues $\Lambda _i \in \mathbb{R}_{<0}, \forall i = 1,\cdots,n$ of $\boldsymbol{A}$. Moreover, referring to \cite{bernstein2009matrix}, it has
	\begin{align}
		\left \| \exp(\boldsymbol{A}t) \right \|  =& \left \| \boldsymbol{Q}\exp(\bm{\Lambda} t) \boldsymbol{Q}^{-1}\right \| \leq \left \| \boldsymbol{Q}\right \|\left \| \exp(\boldsymbol{\Lambda} t) \right \|\left \|\boldsymbol{Q}^{-1}\right \| \nonumber \\
		\leq& \left \| \boldsymbol{Q} \right \| \left \| \boldsymbol{Q}^{-1} \right \| \exp(\bar{\Lambda} t),
	\end{align}
	where $\bar{\Lambda} = \max_{i=1,\cdots,n} \Lambda_i$ is the maximal value of eigenvalues $\Lambda_i$. Considering the prediction error bound in Theorem 7, then \cref{eq_eSolution} is further derived as
	\begin{align}
		\boldsymbol{e}\left(t\right) \leq& \exp\left({\bar{\Lambda}\left(t-t_0\right) \|\boldsymbol{Q}\|\|\boldsymbol{Q}^{-1}\| \|\boldsymbol{e}(t_0)\|}\right) \nonumber \\
		&+ \|\boldsymbol{Q}\|\|\boldsymbol{Q}^{-1}\| \int_{t_0}^{t} \exp\left({\bar{\Lambda}(t-\tau)} \right)\omega(\tau) d\tau\\
		\leq& \exp\left({\bar{\Lambda}\left(t-t_0\right) }\right)\|\boldsymbol{Q}\|\|\boldsymbol{Q}^{-1}\| \|\boldsymbol{e}(t_0)\| \nonumber \\
		&+ \|\boldsymbol{Q}\|\|\boldsymbol{Q}^{-1}\| \bar{\omega} \int_{t_0}^{t} \exp\left({\bar{\Lambda}(t-\tau)}\right) d\tau , \nonumber
	\end{align}
	where $\bar{\omega} = \max_{t\in\mathbb{R}_{\geq0}}{\omega}(t)$. Due to the fact that
	\begin{align}
		\int_{t_{0}}^{t} \exp\left({\bar{\Lambda}(t-\tau)}\right) d \tau =& -\left.\bar{\Lambda}^{-1} \exp\left({\bar{\Lambda}(t-\tau)}\right)\right|_{t_{0}} ^{t} \\
		=& -\bar{\Lambda}^{-1}\left(1-\exp({\bar{\Lambda}\left(t-t_{0}\right)})\right), \nonumber
	\end{align}
	the tracking error bound 
	\begin{align}
		\boldsymbol{e}(t) \leq& \exp\left({\bar{\Lambda}\left(t-t_{0}\right)}\right)\|\boldsymbol{Q}\|\left\|\boldsymbol{Q}^{-1}\right\|\left\|\boldsymbol{e}\left(t_{0}\right)\right\| \\
		&- \left(1-\exp\left({\bar{\Lambda}\left(t-t_{0}\right)}\right)\right) \|\boldsymbol{Q}\|\|\boldsymbol{Q}^{-1}\| \bar{\Lambda}^{-1} \bar{\omega}. \nonumber
	\end{align}
	Therefore, the ultimate tracking error bound is written as
	\begin{align}
		\lim_{t \to \infty} \boldsymbol{e}(t)\le -\|\boldsymbol{Q}\|\|\boldsymbol{Q}^{-1}\|\frac{\bar{\omega}}{\bar{\Lambda}} = \|\boldsymbol{Q}\|\|\boldsymbol{Q}^{-1}\|\frac{\bar{\omega}}{| \bar{\Lambda} |},
	\end{align}
	by considering the Hurwitz $\boldsymbol{A}$ inducing $\bar{\Lambda} < 0$, which concludes the proof.
	
	\section{Lipschitz Continuity of Kernels}\label{sec_LipshitzKernels}
	
	The Lipschitz constant and corresponding distance definition is summarized in \cref{sample-table}, whose results are proven in the following subsections.
	
	\begin{table*}[t]
		\centering
		\begin{tabular}{lccc}
			\hline
			Name &Kernel & Distance function & Lipschitz constants \\
			\hline
			Linear   & $\sigma_l^2 (\bm{x} - \bm{c})^T (\bm{x}' - \bm{c}) + \sigma_b^2$ & $(\bm{x} - \bm{c})^T (\bm{x}' - \bm{c})$& $\sigma_l^2$ \\
			SE & $\sigma_f^2 \exp (-\| \bm{x} - \bm{x}'\|^2 / (2\sigma_l^{2}) )$ & $\|\boldsymbol{x}-\boldsymbol{x}'\|$&  $\sigma_f^2 \sigma_l^{-1}\exp(-0.5)$ \\
			ARD-SE   & $\sigma_f^2 \exp (-(\bm{x} - \bm{x}')^T \bm{\Sigma}_L^{-2} (\bm{x} - \bm{x}') / 2 )$  &  $\| \bm{\Sigma}_L^{-1} (\bm{x} - \bm{x}') \|$& $\sigma_f^2 \exp(-0.5)$ \\
			RQ   & $\sigma_f^2  \left ( 1+  \|\boldsymbol{x}-\boldsymbol{x}'\|^2 / (2\alpha \sigma_l^2) \right )$  & $\|\boldsymbol{x} -\boldsymbol{x}'\|$& $\frac{\sigma_f^2}{\sigma_l^2}\left( 2\alpha (\alpha + 1) / (2\alpha^2+2\alpha-1) \right)^{-\alpha - 1}\sqrt{\frac{2\alpha l^2}{2\alpha^2 + 2\alpha - 1}} $      \\
			Periodic  & $\sigma_f^2 \left ( 1+  \frac{\|\boldsymbol{x}-\boldsymbol{x}'\|^2}{2\alpha \sigma_l^2}\right )$  &$\|\boldsymbol{x}-\boldsymbol{x}'\|$& $\begin{cases}
				\frac{4\pi \sigma_f^2}{p\sigma_l^2} \exp (-\frac{2}{\sigma_l^2}), & \text{if}~ \sigma_l^2 \geq 4 \\
				\frac{2\pi \sigma_f^2}{p\sigma_l} \exp (-\frac{1}{2}),  & \text{if}~ \sigma_l^2 < 4
			\end{cases}$\\
			\hline
		\end{tabular}
		\caption{
			Distance functions and Lipschitz constants of commonly used kernels, e.g., linear (LINEAR) kernel, squared exponential (SE) kernel, automatic relevance determination-squared exponential (ARD-SE) kernel, rational quadratic (RQ) kernel, and periodic (PERIODIC) kernel.
		}
		\label{sample-table}
	\end{table*}
	
	\subsection{Proof of Lemma 4 for SE Kernel}
	
	Consider the distance function as
	\begin{align}
		d_{SE} = \|\boldsymbol{x}-\boldsymbol{x}'\|,
	\end{align}
	the derivative of $\kappa_{SE}$ w.r.t $d_{SE}$ is
	\begin{align}
		\nabla_{d_{SE}\left(x, x^{\prime}\right)} \kappa\left(x, x^{\prime}\right) =& \sigma^{2} \nabla_{d_{SE}\left(x, x^{\prime}\right)} \exp \left(-\frac{d_{x}^{2}\left(x, x^{\prime}\right)}{2 l^{2}}\right) \\
		=&-\sigma^{2} \exp \left(-\frac{d_{x}^{2}\left(x, x^{\prime}\right)}{2 l^{2}}\right) \frac{d_{x}\left(x, x^{\prime}\right)}{l^{2}}. \nonumber
	\end{align}
	Then its second derivative is written as
	\begin{align}\label{eq_secondDerivationSE}
		&\nabla_{d_{SE}\left(x, x^{\prime}\right)} \left(\left|\nabla_{d_{SE}\left(x, x^{\prime}\right)} \kappa\left(x, x^{\prime}\right)\right|\right) \nonumber\\
		=& \nabla_{d_{x}\left(x, x^{\prime}\right)}\left(\sigma^{2} \exp \left(-\frac{d_{x}^{2}\left(x, x^{\prime}\right)}{2 l^{2}}\right) \frac{d_{x}\left(x, x^{\prime}\right)}{l^{2}}\right) \nonumber\\
		=&\left(-\sigma^{2} \exp \left(-\frac{d_{x}^{2}\left(x, x^{\prime}\right)}{2 l^{2}}\right) \frac{d_{x}\left(x, x^{\prime}\right)}{l^{2}} \frac{d_{x}\left(x, x^{\prime}\right)}{l^{2}}\right) \\
		&+\left(\sigma^{2} \exp \left(-\frac{d_{x}^{2}\left(x, x^{\prime}\right)}{2 l^{2}}\right) \frac{1}{l^{2}}\right) \nonumber\\
		=&\frac{\sigma^{2}}{l^{2}} \exp \left(-\frac{d_{x}^{2}\left(x, x^{\prime}\right)}{2 l^{2}}\right)\left(-\frac{d_{x}^{2}\left(x, x^{\prime}\right)}{l^{2}}+1\right) \nonumber
	\end{align}
	As the maximum value of $|\nabla_{d_{SE}\left(x, x^{\prime}\right)} \kappa\left(x, x^{\prime}\right)|$, i.e., its the upper bound, is achieved by letting \cref{eq_secondDerivationSE} to be zero. Therefore, when $\frac{d_{x}^{2}\left(x, x^{\prime}\right)}{l^{2}}=1$, the Lipschitz constant of $\kappa_{SE}$ is obtained as
	\begin{align}
		L_{\kappa,SE} = \frac{\sigma_f^2}{\sigma_l}\exp(-\frac{1}{2}),
	\end{align}
	which concludes the proof.
	
	\subsection{Corollaries of the Other Lipschitz constant of Kernel}
	\subsection{Lipschitz Continuity for ARD-SE Kernel}
	
	The distance function and the Lipshitz constant associated with ARD-SE kernel are given in the following corollary.
	\begin{corollary}\label{corol_ARDSE}
		Consider the ARD-SE kernel used and its distance function defined as
		\begin{align*}
			d_{\text{ARD-SE}}(\bm{x}, \bm{x}') = \| \bm{\Sigma}_L^{-1} (\bm{x} - \bm{x}') \|, ~~ \forall \bm{x}, \bm{x}' \in \mathbb{R}^n,
		\end{align*}
		then the corresponding Lipschitz constant is $L_{\kappa, \text{ARD-SE}} = \sigma_f^2\exp(-0.5)$.
	\end{corollary}
	\begin{proof}
		Define an auxiliary variable $\bm{\varsigma}$ as $\bm{\varsigma} = \bm{\Sigma}_L^{-1} (\bm{x} - \bm{x}')$, then the ARD-SE kernel is reformulated as $\kappa_{\text{ARD-SE}}(\bm{x}, \bm{x}') = \kappa_{\text{SE}}(\bm{\varsigma}) = \sigma_f^2 \exp (-0.5 \bm{\varsigma}^T \bm{\varsigma} )$, such that the derivative of $\kappa_{\text{SE}}(\bm{x}, \bm{x}')$ w.r.t $\bm{x}$ is written as
		\begin{align}
			\frac{\mathrm{d}}{ \mathrm{d} \bm{\varsigma}} \kappa_{\text{SE}}(\bm{\varsigma})= - \kappa_{\text{SE}}(\bm{\varsigma}) \bm{\varsigma}.
		\end{align}
		Moreover, the norm of the derivative is presented as
		\begin{align}
			\| \frac{\mathrm{d}}{ \mathrm{d} \bm{\varsigma}} \kappa_{\text{SE}}(\bm{\varsigma}) \| = \kappa_{\text{SE}}(\| \bm{\varsigma} \|) \| \bm{\varsigma} \|,
		\end{align}
		where the maximum of the right-hand side is achieved when $\| \bm{\varsigma} \| = 1$ with the maximum as
		\begin{align} \label{eqn_Lk_SE}
			\max_{\bm{x} \in \mathbb{R}^n} \| \frac{\mathrm{d}}{ \mathrm{d} \bm{\varsigma}} \kappa_{\text{SE}}(\bm{\varsigma}) \| =& \max_{\bm{\varsigma} \in \mathbb{R}^n} \kappa_{\text{SE}}(\| \bm{\varsigma} \|) \| \bm{\varsigma} \| \\
			=& \sigma_f^2 \exp(-0.5) = L_{\kappa, \text{ARD-SE}}. \nonumber
		\end{align}
		Next, we consider the difference of ARD-SE kernel with different inputs as
		\begin{align}
			\kappa_{\text{ARD-SE}}(\bm{x},\bm{x}_1) - \kappa_{\text{ARD-SE}}(\bm{x},\bm{x}_2) = \kappa_{\text{SE}}(\bm{\varsigma}_1) - \kappa_{\text{SE}}(\bm{\varsigma}_2)
		\end{align}
		with $\bm{\varsigma}_1 = \bm{\Sigma}_L^{-1} (\bm{x} - \bm{x}_1)$ and $\bm{\varsigma}_2 = \bm{\Sigma}_L^{-1} (\bm{x} - \bm{x}_2)$.
		Using the mean value theorem, the absolute value of the difference is written as
		\begin{align}
			&| \kappa_{\text{ARD-SE}}(\bm{x},\bm{x}_1) - \kappa_{\text{ARD-SE}}(\bm{x},\bm{x}_2) | \\
			=& \left| \Big( \frac{\mathrm{d} \kappa_{\text{SE}}(\bm{\varsigma}^*)}{\mathrm{d} \bm{\varsigma}^*} \Big)^{T}  (\bm{\varsigma}_1 - \bm{\varsigma}_2) \right| \le  \Big\| \frac{\mathrm{d} \kappa_{\text{SE}}(\bm{\varsigma}^*)} {\mathrm{d} \bm{\varsigma}^*} \Big\| \| \bm{\varsigma}_1 - \bm{\varsigma}_2 \| \nonumber
		\end{align}
		with $\bm{\varsigma}^* = a \bm{\varsigma}_1 + (1 - a) \bm{\varsigma}_2$ with $a \in [0,1]$.
		Considering \eqref{eqn_Lk_SE} and the fact $\| \bm{\varsigma}_1 - \bm{\varsigma}_2 \| = d_x(\bm{x}_1, \bm{x}_2)$, the inequality
		\begin{align}
			| \kappa_{\text{ARD-SE}}(\bm{x},\bm{x}_1) - \kappa_{\text{ARD-SE}}(\bm{x},\bm{x}_2) | \le L_{\kappa, \text{ARD-SE}} d_x(\bm{x}_1, \bm{x}_2)
		\end{align}
		holds for any $\bm{x}, \bm{x}_1, \bm{x}_2 \in \mathbb{R}^n$, which concludes the proof.
	\end{proof}
	
	\subsection{Lipschitz Continuity for Linear Kernel}
	
	The distance function and the Lipshitz constant associated with Linear kernel are given in the following corollary.
	\begin{corollary}
		Consider the linear kernel defined by
		\begin{align}
			\kappa_{Lin}\left(x,x^\prime\right)=\sigma_l^2(\bm{x} - \bm{c})^T (\bm{x}' - \bm{c})+ \sigma_b^2,
		\end{align}
		where $\sigma_l$ is the hyper-parameter and the constant variance $\sigma_b^2$ determines how far from 0 the height of the function will be at zero. Choose the distance function as
		\begin{align}
			d_{\text{Lin}}(\bm{x}, \bm{x}') = (\bm{x} - \bm{c})^T (\bm{x}' - \bm{c}), ~~ \forall \bm{x}, \bm{x}' \in \mathbb{R}^n,
		\end{align}
		and consider the results in Lemma 3. Then the corresponding Lipschitz constant is written as
		\begin{align}
			L_{\kappa, \text{Lin}} = \sigma_l^2.
		\end{align}
	\end{corollary}
	\begin{proof}
		The proof is straightforward and omitted for brevity.
	\end{proof}
	
	\subsection{Lipschitz Continuity for RQ Kernel}
	
	The distance function and the Lipshitz constant associated with RQ kernel are given in the following corollary.
	
	\begin{corollary}
		Consider the rational quadratic (RQ) kernel defined by
		\begin{align}
			\kappa_{RQ}\left(\boldsymbol{x},\boldsymbol{x}^\prime\right)=\sigma_f^2 \left ( 1+  \frac{\|\boldsymbol{x}-\boldsymbol{x}'\|^2}{2\alpha \sigma_l^2}\right ),
		\end{align}
		where  $\sigma_f$ and $\sigma_l$ are the hyper-parameters and the parameter $\alpha > 0$ governs the relative weighting of large-scale and small-scale variations. As $\alpha$ tends towards infinity, the RQ kernel converges to the SE kernel.
		Choose the distance function as
		\begin{align}
			d_{\text{RQ}}(\bm{x}, \bm{x}') = \|\boldsymbol{x}-\boldsymbol{x}'\|, ~~ \forall \bm{x}, \bm{x}' \in \mathbb{R}^n,
		\end{align}
		and consider the results in Lemma 3. Then the corresponding Lipschitz constant is written as
		\begin{align}
			L_{\kappa, \text{RQ}} = \frac{\sigma_f^2}{\sigma_l^2}\left(\frac{2\alpha^2+2\alpha}{2\alpha^2+2\alpha-1}\right)^{-\alpha-1}\sqrt{\frac{2\alpha l^2}{2\alpha^2+2\alpha-1}}.
		\end{align}
	\end{corollary}
	\begin{proof}
		The proof follows \cref{corol_ARDSE}. We first derive the derivative of $\kappa_{RQ}$ w.r.t $d_{\text{RQ}}$ as 
		\begin{align}
			&\nabla_{d_{RQ}\left(\boldsymbol{x},\boldsymbol{x}^\prime\right)}\kappa_{RQ}\left(\boldsymbol{x},\boldsymbol{x}^\prime\right) \\
			=&-\sigma^2\left(1+\frac{d_{RQ}^2\left(\boldsymbol{x},\boldsymbol{x}^\prime\right)}{2\alpha \sigma_l^2}\right)^{-\alpha-1}\left(\frac{d_{RQ}\left(\boldsymbol{x},\boldsymbol{x}^\prime\right)}{\sigma_l^2}\right). \nonumber
		\end{align}
		According to the second derivative 
		\begin{align}
			&\nabla_{d_{RQ}\left(\boldsymbol{x},\boldsymbol{x}^\prime\right)} \left(\left|\nabla_{d_{RQ}\left(\boldsymbol{x},\boldsymbol{x}^\prime\right)}\kappa_{RQ}\left(\boldsymbol{x},\boldsymbol{x}^\prime\right)\right|\right) \nonumber \\
			=&\nabla_{d_{RQ} \left(\boldsymbol{x}, \boldsymbol{x}^\prime\right)} \left(\sigma_f^2\left(1 \!+\! \frac{d_{RQ}^2\left(\boldsymbol{x},\boldsymbol{x}^\prime\right)}{2\alpha \sigma_l^2}\right)^{-\alpha-1} \!\! \left(\frac{d_{RQ}\left(\boldsymbol{x},\boldsymbol{x}^\prime\right)}{\sigma_l^2}\right)\right)\nonumber\\
			=&\left(-\left(\alpha+1\right)\frac{\sigma_f^2}{\sigma_l^2}\left(1+\frac{d_{RQ}^2\left(\boldsymbol{x},\boldsymbol{x}^\prime\right)}{2\alpha \sigma_l^2}\right)^{-\alpha-2} \!\! \left(\frac{d_{RQ}^2\left(\boldsymbol{x},\boldsymbol{x}^\prime\right)}{\sigma_l^2}\right)\right) \nonumber \\
			&+ \left(\frac{\sigma_f^2}{\sigma_l^2}\left(1+\frac{d_{RQ}^2\left(\boldsymbol{x},\boldsymbol{x}^\prime\right)}{2\alpha \sigma_l^2}\right)^{-\alpha-1}\right) \\
			=&\frac{\sigma_f^2}{\sigma_l^2}\left(1+\frac{d_{RQ}^2\left(\boldsymbol{x},\boldsymbol{x}^\prime\right)}{2\alpha \sigma_l^2}\right)^{-\alpha-1} \nonumber \\
			&\times \left(1-\left(\alpha+1\right)\left(1+\frac{d_{RQ}^2\left(\boldsymbol{x},\boldsymbol{x}^\prime\right)}{2\alpha \sigma_l^2}\right)^{-1}\left(\frac{d_{RQ}^2\left(\boldsymbol{x},\boldsymbol{x}^\prime\right)}{\sigma_l^2}\right)\right), \nonumber
		\end{align}
		then the maximum of $\left|\nabla_{d_{RQ}\left(\boldsymbol{x},\boldsymbol{x}^\prime\right)}\kappa_{RQ}\left(\boldsymbol{x},\boldsymbol{x}^\prime\right)\right|$ is achieved when 
		\begin{align}
			\left(\alpha+1\right)\left(1+\frac{d_{RQ}^2\left(\boldsymbol{x},\boldsymbol{x}^\prime\right)}{2\alpha \sigma_l^2}\right)^{-1}\left(\frac{d_{RQ}^2\left(\boldsymbol{x},\boldsymbol{x}^\prime\right)}{\sigma_l^2}\right)=1.
		\end{align}
		Therefore, 
		\begin{align}
			L_{\kappa,RQ}=\frac{\sigma_f^2}{\sigma_l^2}\left(\frac{2\alpha^2+2\alpha}{2\alpha^2+2\alpha-1}\right)^{-\alpha-1}\sqrt{\frac{2\alpha \sigma_l^2}{2\alpha^2+2\alpha-1}}.
		\end{align}
		This concludes the proof.
	\end{proof}
	
	\subsection{Lipschitz Continuity for Periodic Kernel}
	
	The distance function and the Lipshitz constant associated with periodic kernel are given in the following corollary.
	
	\begin{corollary}
		Consider the periodic kernel defined by
		\begin{align}
			\kappa_{Per}\left(x,x^\prime\right)=\sigma_f^2 \exp \left ( - \frac{2\sin\left ( \frac{\pi\|\boldsymbol{x}-\boldsymbol{x}'\|}{p} \right ) }{\sigma_l^2} \right ),
		\end{align}
		where  $\sigma_f$ and $\sigma_l$ are the hyper-parameters and the period $p>0$ determines the distance between repetitions of the function.
		Choose the distance function as
		\begin{align}
			d_{\text{Per}}(\bm{x}, \bm{x}') = \|\boldsymbol{x}-\boldsymbol{x}'\|, ~~ \forall \bm{x}, \bm{x}' \in \mathbb{R}^n,
		\end{align}
		and consider the results in Lemma 3. Then the corresponding Lipschitz constant is written as
		\begin{align}
			L_{\kappa, \text{Per}} = \begin{cases}
				\frac{4\pi \sigma_f^2}{p\sigma_l^2} \exp (-\frac{2}{\sigma_l^2}), & \text{ if } \sigma_l^2 \geq 4 \\
				\frac{2\pi \sigma_f^2}{p\sigma_l} \exp (-\frac{1}{2}),  & \text{ if } \sigma_l^2 < 4
			\end{cases}.
		\end{align}
	\end{corollary}
	\begin{proof}
		Similar to \cref{corol_ARDSE}, the derivative of $\kappa_{Per}$ w.r.t $d_{\text{Per}}$ is
		\begin{align}\label{eq_derivationPerKernel}
			\nabla_{d_x\left(\boldsymbol{x},\boldsymbol{x}^\prime\right)}&\kappa_{Per}\left(\boldsymbol{x},\boldsymbol{x}^\prime\right) = - 4\sin{\left(\pi\frac{d_{Per}\left(\boldsymbol{x},\boldsymbol{x}^\prime\right)}{p}\right)}  \\
			&\times \left(\frac{\pi\sigma_f^2}{p\sigma_l^2}\right) \exp{\left(-\frac{2\sin^2{\left(\pi\frac{d_{Per}\left(\boldsymbol{x},\boldsymbol{x}^\prime\right)}{p}\right)}}{\sigma_l^2}\right)}. \nonumber
		\end{align}
		To facilitate the derivation, we define $s(\boldsymbol{x},\boldsymbol{x}') = \sin\left ( \pi\|\boldsymbol{x}-\boldsymbol{x}'\| / p \right)$. Then the derivative of \cref{eq_derivationPerKernel} w.r.t $s(\boldsymbol{x},\boldsymbol{x}')$ is written as
		\begin{align}
			&\nabla_{s\left(\boldsymbol{x}, \boldsymbol{x}^{\prime}\right)}\left(\left|\nabla_{d_{Per}\left(\boldsymbol{x}, \boldsymbol{x}^{\prime}\right)} \kappa_{Per}\left(\boldsymbol{x}, \boldsymbol{x}^{\prime}\right)\right|\right)\nonumber\\
			=&\nabla_{\sin \left(\pi \frac{\left\|\boldsymbol{x}-\boldsymbol{x}^{\prime}\right\|}{p}\right)}\left(\exp \left(-\frac{2 s^{2}\left(\boldsymbol{x}, \boldsymbol{x}^{\prime}\right)}{\sigma_l^{2}}\right) 4 s\left(\boldsymbol{x}, \boldsymbol{x}^{\prime}\right)\left(\frac{\pi \sigma_f^{2}}{p \sigma_l^{2}}\right)\right) \nonumber \\
			=&\left(-\frac{4 \pi \sigma_f^{2}}{p \sigma_l^{2}} \exp \left(-\frac{2 s^{2}\left(\boldsymbol{x}, \boldsymbol{x}^{\prime}\right)}{l^{2}}\right) \frac{4 s\left(\boldsymbol{x}, \boldsymbol{x}^{\prime}\right)}{\sigma_l^{2}} s\left(\boldsymbol{x}, \boldsymbol{x}^{\prime}\right)\right) \\
			&+\left(\frac{4 \pi \sigma_f^{2}}{p \sigma_l^{2}} \exp \left(-\frac{2 s^{2}\left(\boldsymbol{x}, \boldsymbol{x}^{\prime}\right)}{\sigma_l^{2}}\right)\right) \nonumber \\
			=& \frac{4 \pi \sigma_f^{2}}{p \sigma_l^{2}} \exp \left(-\frac{2 s^{2}\left(\boldsymbol{x}, \boldsymbol{x}^{\prime}\right)}{\sigma_l^{2}}\right)\left(-\frac{4 s^{2}\left(\boldsymbol{x}, \boldsymbol{x}^{\prime}\right)}{\sigma_l^{2}}+1\right). \nonumber
		\end{align}
		If $\frac{\sigma_l^2}{4}\geq1$, one has
		\begin{align}
			1-\frac{4s(\boldsymbol{x},\boldsymbol{x}')}{\sigma_l^2} \geq 1-\frac{4}{\sigma_l^2} \geq 0
		\end{align}
		due to the fact $s(\boldsymbol{x},\boldsymbol{x}') \geq 1$. Therefore, the maximum of $|\nabla_{d_{Per}\left(\boldsymbol{x}, \boldsymbol{x}^{\prime}\right)}\kappa_{Per}(\boldsymbol{x}, \boldsymbol{x}^{\prime})|$ locating at $s(\boldsymbol{x},\boldsymbol{x}')=1$ is 
		\begin{align}
			\max{\left|\nabla_{d_{Per}\left(\boldsymbol{x},\boldsymbol{x}^\prime\right)}\kappa_{Per}\left(\boldsymbol{x},\boldsymbol{x}^\prime\right)\right|}=\frac{4\pi\sigma_f^2}{p\sigma_l^2}\exp{\left(-\frac{2}{\sigma_l^2}\right)}.
		\end{align}
		If $\frac{\sigma_l^2}{4}<1$, then the maximum of $|\nabla_{d_{Per}\left(\boldsymbol{x}, \boldsymbol{x}^{\prime}\right)}\kappa_{Per}(\boldsymbol{x}, \boldsymbol{x}^{\prime})|$ locating at $s(\boldsymbol{x},\boldsymbol{x}')=\frac{l}{2}$ is 
		\begin{align}
			\max{\left|\nabla_{d_{Per}\left(x,x^\prime\right)}\kappa_{Per}\left(\boldsymbol{x},\boldsymbol{x}^\prime\right)\right|}=\frac{2\pi\sigma_f^2}{p\sigma_l}\exp{\left(-\frac{1}{2}\right)},
		\end{align}
		which concludes the proof.
	\end{proof}
	
	%%%%%%%%%%%%%%%%%%%%%%%%%%%%%%%%%%%%%%%%%%%%%%%%%%%%%%%%%%%%%%%%%%%%%%%%%
	\section{Additional Simulation Results}\label{sec_additionalSimulation}
	
	\subsection{Simulation Configuration}\label{subsec_simulationConfiguration}
	The codes are executed in MATLAB R2021b on the laptop with AMD Ryzen 7 5800H with 16.0GB RAM using MediaTek Wi-Fi 6 MT7921 Wireless LAN Card.

	\subsection{Regression Tasks}\label{subsec_regressionTasks}
	
	In the GPgym platform, we employ a default configuration featuring 4 Gaussian Process (GP) models, with their parameters detailed in \cref{table_LoGGP_parameters}. 
	The time interval for the exchange of input $\bm{x}$ and the reception of predictions is set at $0.02$ seconds to and from all distributed GPs with the centralized node. 
	We utilize the ARD-SE kernel with optimized hyperparameters obtained from the full dataset. 
	Moreover, the listening frequency for all GP models is established at $1000$ Hz, and predictions are promptly relayed back to the centralized node upon generation.
	
	\begin{table*}[t]
		\centering
		\begin{tabular}{p{0.3\linewidth} llll}
			\hline
			Properties & GP model 1 & GP model 2 & GP model 3 & GP model 4 \\
			\hline
			Maximal Number of data samples  	& $100$ 	& $50$ 		& $500$ 	& $1000$ \\
			Maximal Number of local GP models		& $100$ 	& $200$ 	& $50$ 		& $20$ \\
			Overlap Rate 							& $0.01$ 	& $0.01$ 	& $0.01$ 	& $0.01$ \\
			\hline
		\end{tabular}
		\caption{Parameters for distributed LoG-GP} \label{table_LoGGP_parameters}
	\end{table*}
	
	In the first example, we set the information set threshold $\bar{\mathfrak{I}}$ to be $4$. 
	The objective is to assess the performance of existing approaches in asynchronous prediction. 
	This is particularly relevant as current approaches solely aggregate the most recent results from the GP models. 
	Consequently, the information set consistently comprises $4$ results sets, i.e., $\mathbb{P}_i ~(i=1,\cdots4)$. 
	Moreover, we provide the additional result of the delay time for the datasets SARCOS and PUMADYN32NM in \cref{fig_DelayTime_sarcos} and \ref{fig_DelayTime_pum32}, respectively.
	
	\begin{figure}[t]
		\centering
		\includegraphics[width=0.48\textwidth]{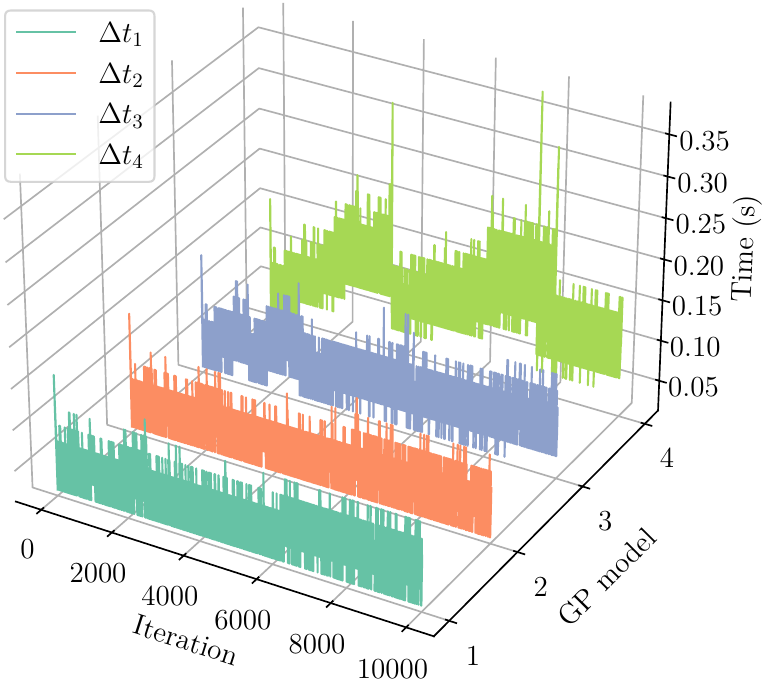}
		\caption{The delay time of the results of GP models.}
		\label{figure_timedelay_kin40k_gp}
	\end{figure}
	
	In \cref{figure_timedelay_kin40k_gp}, the temporal delay $\Delta t_i$ relative to the current time of used results of GP model $i~(i=1,2,3,4)$ on the KIN40K dataset is shown. 
	It is noteworthy that, given the existing methods only utilize the most recent results, the delay time $\Delta t_i$ is inherently identical to the delay time of results in the information set. 
	Conversely, with AsyncDGP, where the results in the information set are composed of previous iterations, the sorted delay time of previous results is presented in \cref{figure_timedelay_kin40k_agg} for the case where $\bar{\mathfrak{I}}= 4$. 
	Specially, during the interval wherein the delay time of used results is consistently increasing, each model incorporates previous results in its utilization.
	
	\begin{figure}[t]
		\centering
		\includegraphics[width=0.48\textwidth]{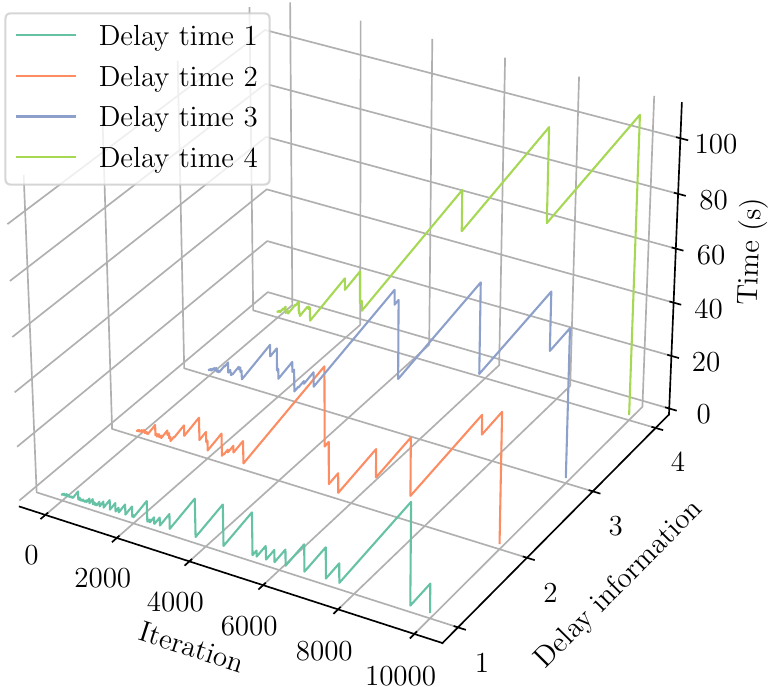}
		\caption{The sorted delay time of the results in the information set of AsyncDGP.}
		\label{figure_timedelay_kin40k_agg}
	\end{figure}
	
	\begin{figure}[t]
		\centering
		\begin{subfigure}{0.48\textwidth}
			\centering
			\includegraphics[width=\textwidth]{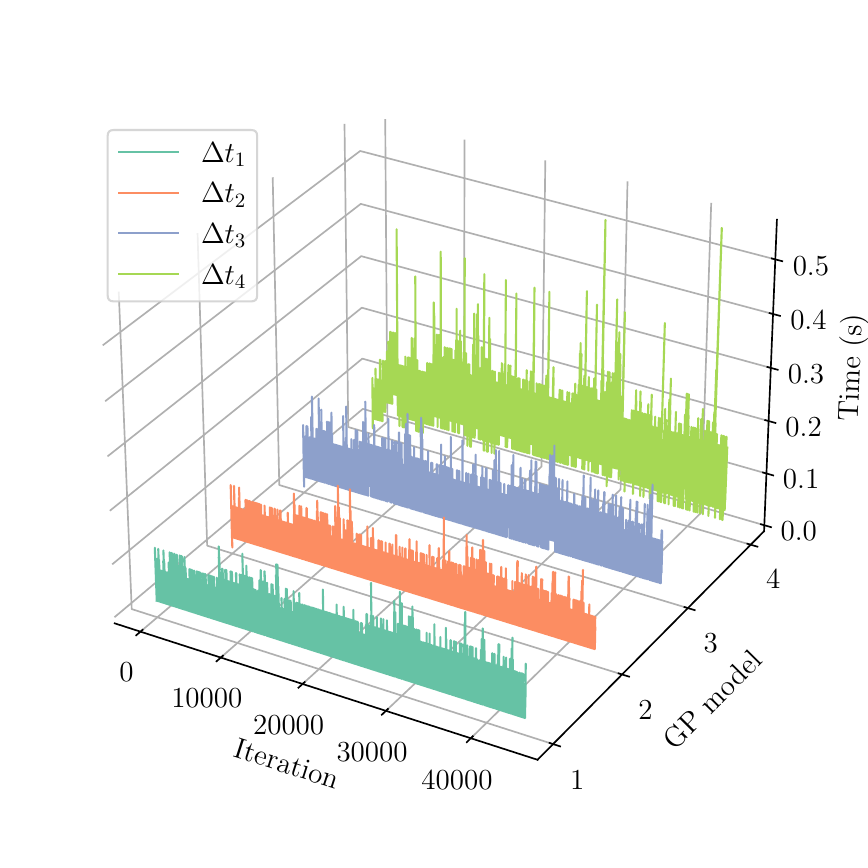}
			\caption{The delay time of the results of GP models.}
		\end{subfigure}
		\begin{subfigure}{0.48\textwidth}
			\centering
			\includegraphics[width=\textwidth]{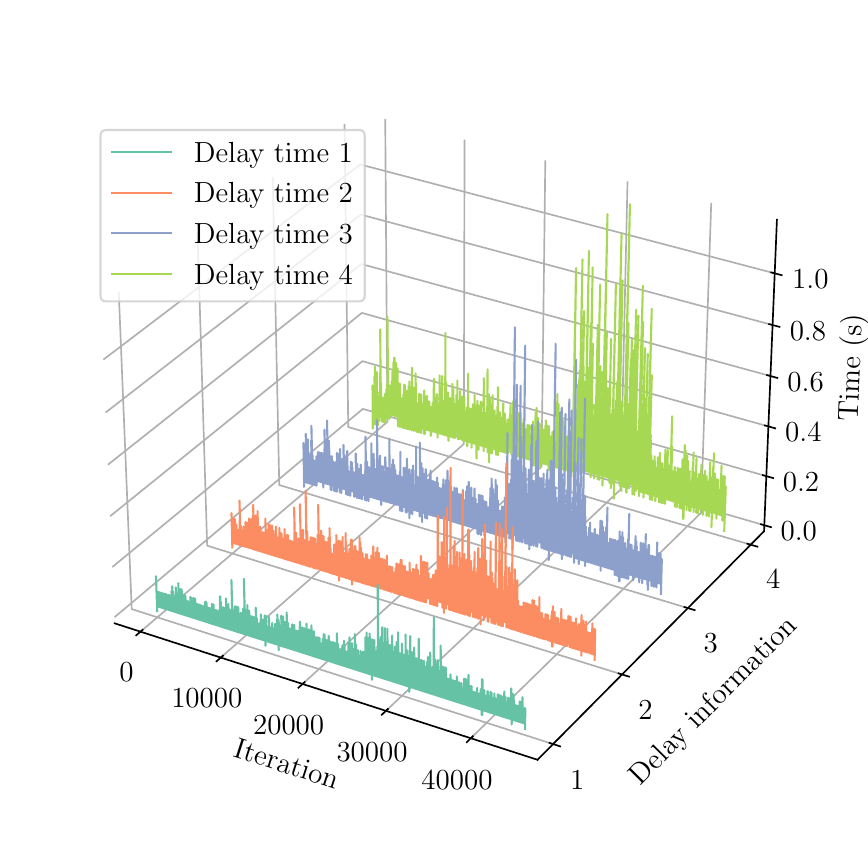}
			\caption{The delay time of the results in the information set.}
		\end{subfigure}
		\caption{The delay time of the results on SARCOS.}%
		\label{fig_DelayTime_sarcos}%
	\end{figure}
	
	\begin{figure}[t]
		\centering
		\begin{subfigure}{0.48\textwidth}
			\centering
			\includegraphics[width=\textwidth]{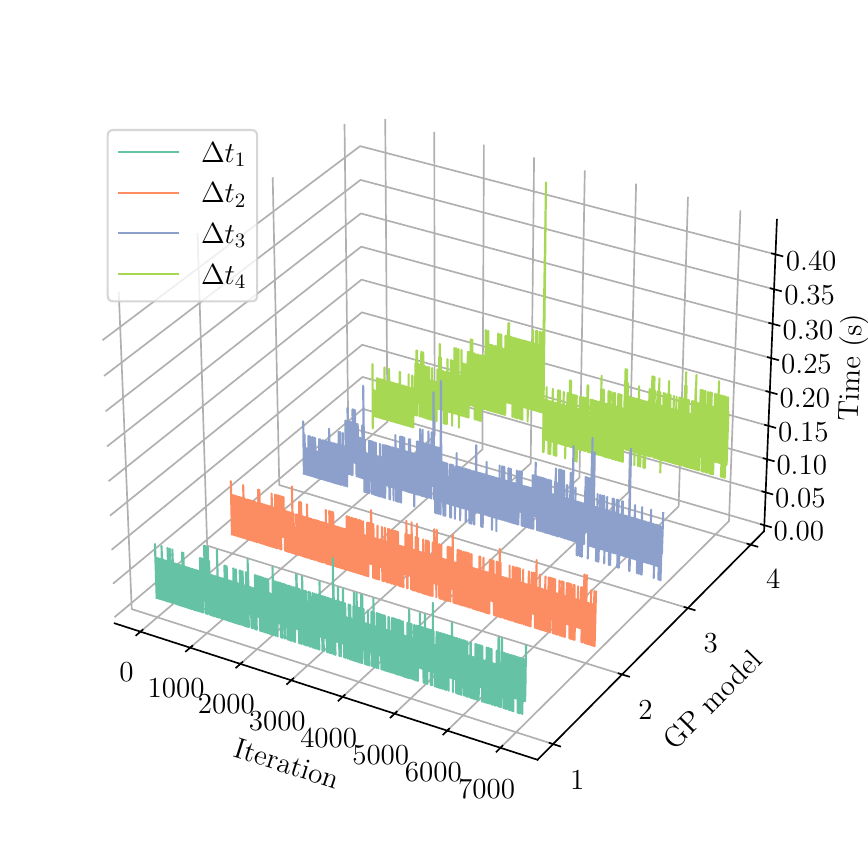}
			\caption{The delay time of the results of GP models.}
		\end{subfigure}
		\begin{subfigure}{0.48\textwidth}
			\centering
			\includegraphics[width=\textwidth]{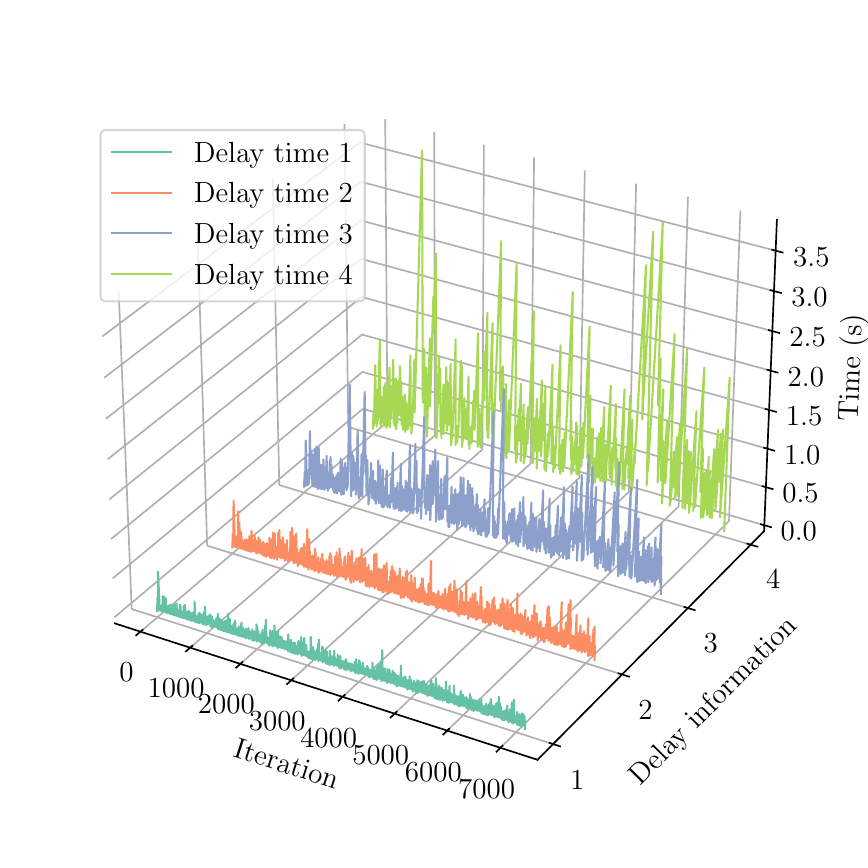}
			\caption{The delay time of the results in the information set.}
		\end{subfigure}
		\caption{The delay time of the results on PUMADYN32NM.}%
		\label{fig_DelayTime_pum32}%
	\end{figure}
	
	Additionally, we present a scenario with $\bar{\mathfrak{I}}=10$, wherein the central node can aggregate more results to infer the unknown functions. To align with this objective and enable similar functionality for other approaches, we have devised the aggregation strategies for the $5$ existing distributed methods, which are formulated as follows.
	\begin{itemize}
		\item Bayesian committee machine (BCM):
		\begin{align}
			\hat{f}(\bm{x}(t)) = \sum_{i = 1}^M \sum_{k = 0}^{\bar{k}_i(t)} \omega_k^i(t) \mu_i(\bm{x}(t_i^{k})) + \omega_m(t) m(\bm{x}(t)),
		\end{align}
		with the aggregation weights 
		\begin{align}
			&\omega_k^i(t) = \omega^2(t) \sigma_i^{-2}(\bm{x}(t_i^{k})), \\
			&\omega_m(t) = \omega^2(t) (1 - \rho(t)) \sigma_f^{-2}, \nonumber
		\end{align}
		where 
		\begin{align}
			&\rho(t) = \sum_{i = 1}^M (\bar{k}_i(t) + 1), \\
			&\omega^{-2}(t) = \sum_{i = 1}^M \sum_{k = 0}^{\bar{k}_i(t)} \sigma_i^{-2}(\bm{x}(t_i^{k})) + (1 - \rho(t)) \sigma_f^{-2}. \nonumber
		\end{align}
		\item Robust Bayesian committee machine (rBCM):
		\begin{align}
			\hat{f}(\bm{x}(t)) = \sum_{i = 1}^M \sum_{k = 0}^{\bar{k}_i(t)} \omega_k^i(t) \mu_i(\bm{x}(t_i^{k})) + \omega_m(t) m(\bm{x}(t)),
		\end{align}
		with the aggregation weights 
		\begin{align}
			&\omega_k^i(t) = \omega^2(t) \rho_k^i(t) \sigma_i^{-2}(\bm{x}(t_i^{k})), \\
			&\omega_m(t) = \omega^2(t) (1 - \rho(t)) \sigma_f^{-2}, \nonumber
		\end{align}
		where $\rho(t) = \sum_{i = 1}^M \sum_{k = 0}^{\bar{k}_i(t)} \rho_k^i(t)$ and
		\begin{align}
			\rho_k^i(t) =&  \log( \sigma_f / \sigma_i(\bm{x}(t_i^{k})), ~~\omega^{-2}(t) \\
			=& \sum_{i = 1}^M \sum_{k = 0}^{\bar{k}_i(t)} \rho_k^i(t) \sigma_i^{-2}(\bm{x}(t_i^{k})) + (1 - \rho(t)) \sigma_f^{-2}. \nonumber
		\end{align}
		\item Product-of-experts (POE):
		\begin{align}
			\hat{f}(\bm{x}(t)) = \sum_{i = 1}^M \sum_{k = 0}^{\bar{k}_i(t)} \omega_k^i(t) \mu_i(\bm{x}(t_i^k)), 
		\end{align}
		with the aggregation weights $\omega_k^i(t) = \omega^2(t) \sigma_i^{-2}(\bm{x}(t_i^k))$, where $\omega^{-2}(t) = \sum_{i = 1}^M \sum_{k = 0}^{\bar{k}_i(t)} \sigma_i^{-2}(\bm{x}(t_i^{k}))$.
		\item Generalized product-of-experts (gPOE):
		\begin{align}
			\hat{f}(\bm{x}(t)) = \sum_{i = 1}^M \sum_{j = 0}^{\bar{k}_i(t)} \omega_k^i(t) \mu_i(\bm{x}(t_i^{k})),
		\end{align}
		with the aggregation weights $\omega_k^i(t)$ defined as $\omega_k^i(t) = \omega^2(t) \rho_k^i(t) \sigma_i^{-2}(\bm{x}(t_i^{k}))$, where
		\begin{align}
			&\rho_k^i(t) =  \log( \sigma_f / \sigma_i(\bm{x}(t_i^k)), \\
			&\omega^{-2}(t) = \sum_{i = 1}^M \sum_{k = 0}^{\bar{k}_i(t)} \rho_k^i(t) \sigma_i^{-2}(\bm{x}(t_i^{k})). \nonumber
		\end{align}
		\item Mixture of experts (MOE):
		\begin{align}
			\hat{f}(\bm{x}(t)) = \frac{\sum_{i = 1}^M \sum_{k = 0}^{\bar{k}_i(t)} \mu_i(\bm{x}(t_i^k))}{\sum_{i = 1}^M \sum_{k = 0}^{\bar{k}_i(t)} 1}. 
		\end{align}
	\end{itemize}
	
	%%%%%%%%%%%%%%%%%%%%%%%%%%%%%%%%%%%%%%%%%%%%%%%%%%%%%%
	\subsection{Control Tasks}\label{subsec_controlTasks}
	In the control task, a total of 1000 training samples are uniformly distributed in the domain $[-3,3] \times [-3,3]$. We conduct $100$ simulations over a duration of $20$ seconds each, employing state broadcasting and listening with a fixed time interval of $0.02$ seconds. The initial states are evenly sampled in $[0,1] \times [0,1]$, and the reference follows a uniform distribution between $[0.4,0.6]$ and $[3,5]$ for each simulation, respectively.
	To demonstrate the control performance during over 100 Monte-Carlo simulations, \cref{fig_controlStates} shows the mean values (solid line) and its variance (light shade) in the 1000 iterations. Moreover, the statistical results are shown in \cref{fig_controlStateMeanMax}.
	
	\begin{figure*}[t]
		\centering
		\includegraphics[width=\textwidth]{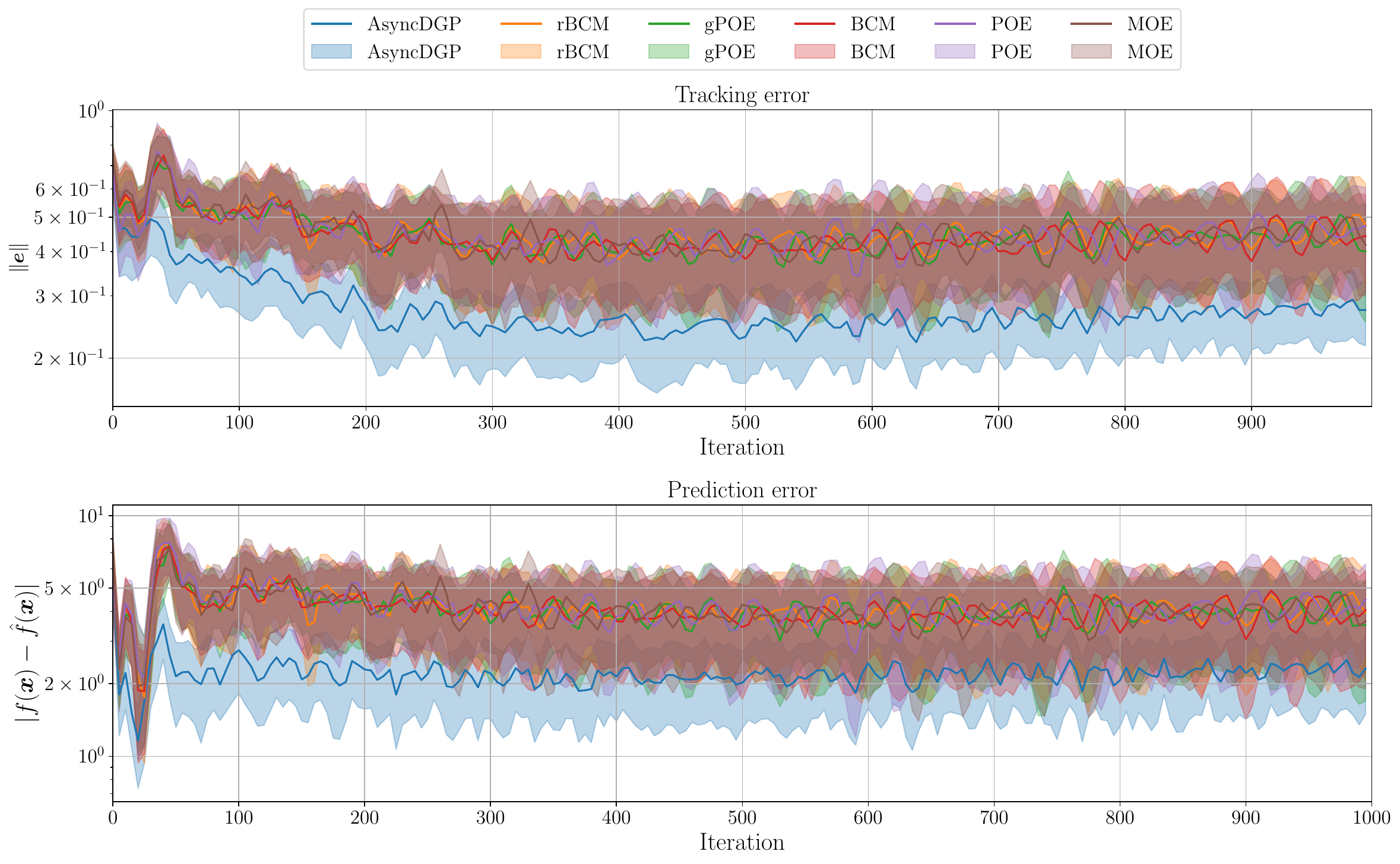}
		\caption{The prediction errors and tracking errors over 1000 iterations.}
		\label{fig_controlStates}
	\end{figure*}
	
	\begin{figure*}[t]
		\centering
		\begin{subfigure}{0.48\textwidth}
			\centering
			\includegraphics[width=\textwidth]{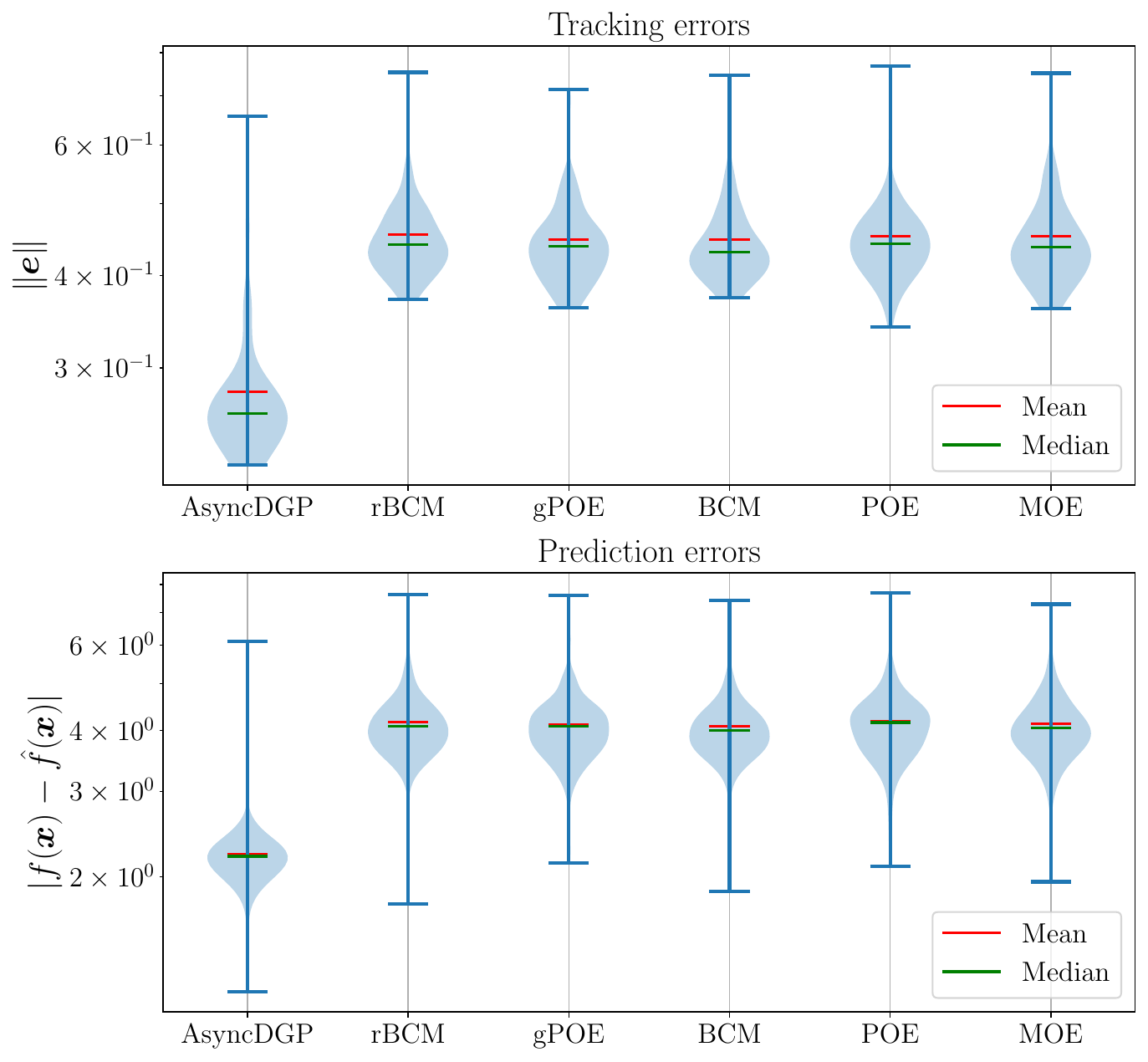}
			\caption{The mean values over 100 Monte-Carlo simulations.}
		\end{subfigure}
		\begin{subfigure}{0.48\textwidth}
			\centering
			\includegraphics[width=\textwidth]{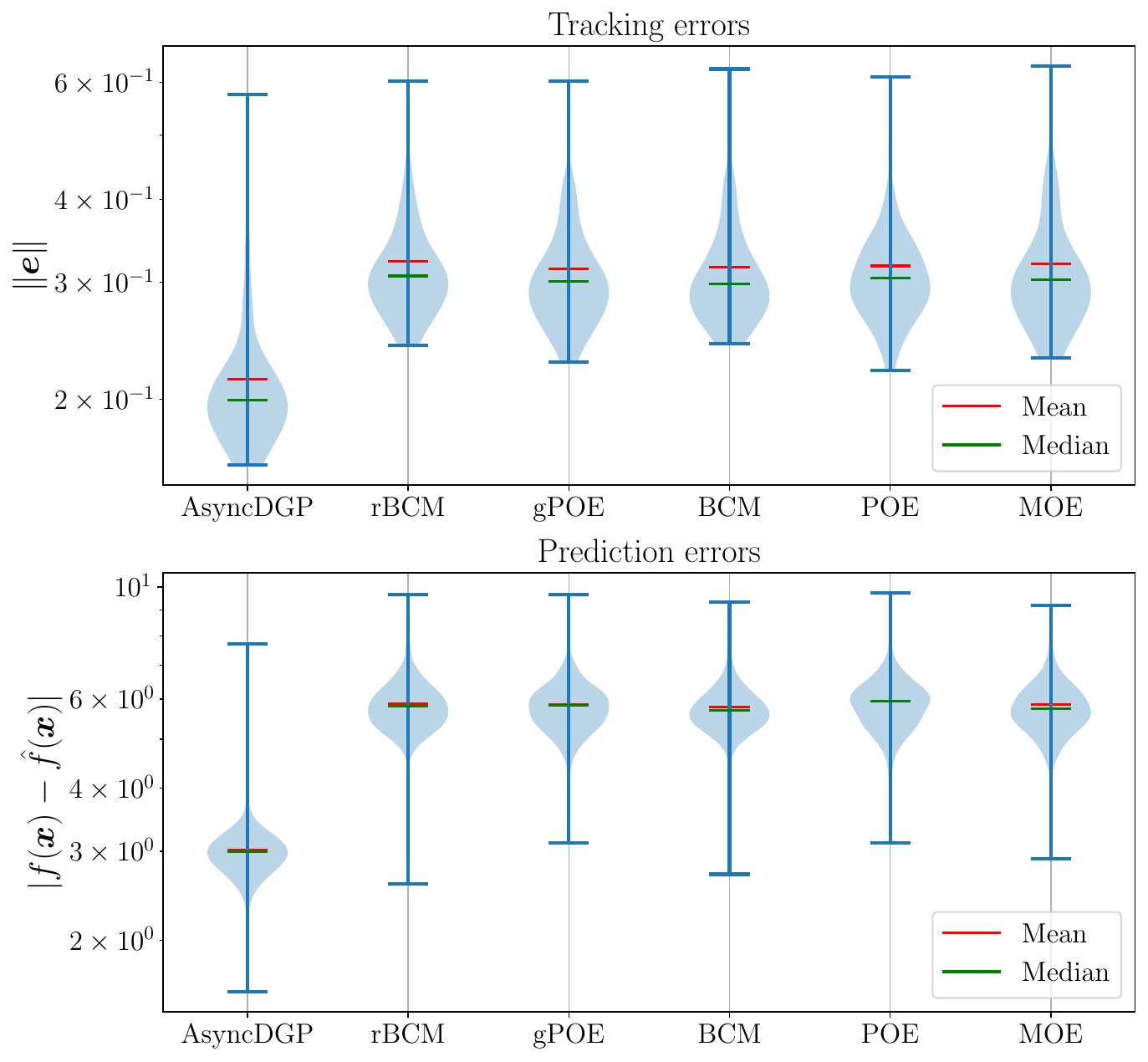}
			\caption{The maximum values over 100 Monte-Carlo simulations.}
		\end{subfigure}
		\caption{The violin plots for mean and maximum values of the Monte-Carlo simulation.}%
		\label{fig_controlStateMeanMax}%
	\end{figure*}
	
	\section*{Acknowledgments}
	
	This work has been financially supported by the Federal Ministry of Education and Research of Germany in the programme of ``Souverän. Digital. Vernetzt.'' under joint project 6G-life with project identification number: 16KISK002, and by the European Research Council (ERC) Consolidator Grant ``Safe data-driven control for human-centric systems (CO-MAN)'' under grant agreement number 864686.
	
	\bibliographystyle{IEEEtran}
	\bibliography{refs}
	
\end{document}